\documentclass[letterpaper]{article}
\usepackage{aaai}
\usepackage{times}
\usepackage{helvet}
\usepackage{courier}
\usepackage{algpseudocode}
\usepackage[ruled]{algorithm}
\usepackage{url}
\usepackage{framed}
\usepackage{amsfonts,amsmath,amsthm,amssymb}
\usepackage{graphicx}
\usepackage{url}
\usepackage{color}

\newcommand {\IE} {\ensuremath {\mathbb{E}}}

\newtheorem{dfn}{Definition}
\newtheorem{thm}{Theorem}

\newtheorem{crl}{Corollary}

\title{MCTS Based on Simple Regret}
\author {David Tolpin, Solomon Eyal Shimony \\
Department of Computer Science, \\
Ben-Gurion University of the Negev, Beer Sheva, Israel \\
\{tolpin,shimony\}@cs.bgu.ac.il}

\title{MCTS Based on Simple Regret}

\begin{document}

\maketitle

\begin{abstract}
UCT, a state-of-the art algorithm for Monte Carlo tree search (MCTS)
in games and Markov decision processes, is based on UCB, a sampling
policy for the Multi-armed Bandit problem (MAB) that 
minimizes the cumulative regret.  However, search differs from MAB in
that in MCTS it is usually only the final ``arm pull'' (the actual
move selection) that collects a reward, rather than all ``arm pulls''.
Therefore, it makes more sense to minimize the simple regret, as
opposed to the cumulative regret. We begin by introducing policies for
multi-armed bandits with lower finite-time and asymptotic simple
regret than UCB, using it to develop a two-stage scheme (SR+CR) for MCTS
which outperforms UCT empirically.

Optimizing the sampling process is itself a
metareasoning problem, a solution of which can use value of
information (VOI) techniques.  Although the theory of VOI for search
exists, applying it to MCTS is non-trivial, as typical myopic
assumptions fail. Lacking a complete working VOI theory for MCTS, we
nevertheless propose a sampling scheme that is ``aware'' of VOI,
achieving an algorithm that in empirical evaluation outperforms 
both UCT and the other proposed algorithms.
\end{abstract}

\section{Introduction}

Monte-Carlo tree search, and especially a version based on the
UCT formula \cite{Kocsis.uct} appears in numerous search applications,
such as \cite{GellyWang.mogo,Eyerich.ctp}. Although these methods are shown to be successful empirically,
most authors appear to be using the UCT formula ``because it has been shown
to be successful in the past'', and ``because it does a good job of
trading off exploration and exploitation''. While the latter statement may be
correct for the multi-armed bandit and for the UCB method \cite{Auer.ucb},
we argue that it is inappropriate for search. The problem is not that
UCT does not work; rather, a simple reconsideration from basic
principles can result in schemes that outperform UCT.

The core issue is that in adversarial search
and search in ``games against nature'' --- optimizing behavior under
uncertainty, the goal is typically to either find a good (or optimal)
strategy, or even just to find the best first action of such a policy. Once
such an action is discovered, it is usually not beneficial to further sample
that action, ``exploitation'' is thus meaningless for search
problems. Finding a good first action is closer to the pure
exploration variant, as seen in the selection problem
\cite{Bubeck.pure,TolpinShimony.blinkered}. In the selection problem,
it is much better to minimize the \emph{simple} regret.  However, the
simple and the cumulative regret cannot be minimized simultaneously;
moreover, \cite{Bubeck.pure} shows that in many cases the smaller the
cumulative regret, the greater the simple regret.

We begin with background definitions and related work. 
Some sampling schemes are introduced, and shown to have better
bounds for the simple regret on sets than UCB, the first contribution of this paper. 
The results are applied to sampling in trees by combining the proposed sampling schemes on the
first step of a rollout with UCT for the rest of the
rollout. An additional sampling scheme based on metareasoning 
principles is also suggested, another contribution
of this paper.
Finally, the performance of the
proposed sampling schemes is evaluated on sets of Bernoulli arms, in randomly
generated 2-level trees, and on the sailing domain, showing where the proposed schemes have improved
performance.

\section{Background and Related Work}
\label{sec:related-work}

Monte-Carlo tree search was initially suggested as a scheme for
finding approximately optimal policies for Markov Decision Processes
(MDP).  An MDP is defined by the set of states $S$, the set of actions
$A$ (also called {\em moves} in this paper),
the transition distribution $T(s, a, s')$, the reward function $R(s, a,
s')$, the initial state $s$ and an optional goal state $t$: $(S, A, T, R, s,
t)$ \cite{Russell.aima}.  Several MCTS schemes explore an MDP by
performing \emph{rollouts}---trajectories from the current state to a
state in which a termination condition is satisfied (either the goal
state, or a cutoff state for which the reward is evaluated
approximately).

\subsection{Multi-armed bandits and UCT}

In the Multi-armed Bandit problem \cite{Vermorel.bandits} we have a set
of $K$ arms (see Figure~\ref{fig:mab-simple-regret}.a). Each arm can
be pulled multiple times. Sometimes a cost is
associated with each pulling action. When the $i$th arm
is pulled, a random reward $X_i$ from an unknown stationary
distribution is encountered.  The reward is usually bounded between 0 and 1.
In the cumulative setting (the focus of much of the research literature on Multi-armed bandits),
all encountered rewards are collected by the agent. 
The UCB scheme was shown to be near-optimal in this respect \cite{Auer.ucb}:

\begin{dfn} Scheme $\mathbf{UCB(c)}$ pulls arm $i$ that maximizes 
upper confidence bound $b_i$ on the reward:
\begin{equation}
b_i=\overline X_i+\sqrt {\frac {c \log (n)} {n_i}}
\label{eqn:ucb}
\end{equation}
where $\overline X_i$ is the average sample reward obtained from arm $i$,
$n_i$ is the number of times arm $i$ was pulled, and $n$ is the total
number of pulls so far. \end{dfn} 

The UCT algorithm, an extension of UCB
to Monte-Carlo Tree Search is described in \cite{Kocsis.uct}, and
shown to outperform many state of the art search algorithms in both
MDP and adversarial games \cite{Eyerich.ctp,GellyWang.mogo}. 

In the simple regret (selection) setting, the agent gets to
collect only the reward of the last pull.

\begin{dfn}
The \textbf{simple regret} $\IE r$ of a sampling policy for the Multi-armed Bandit
Problem is the expected difference between the best true expected reward
$\mu_*$ and the true expected reward $\mu_j$ of the arm with the greatest sample mean,
$j =\arg \max_i\overline X_i$:
\begin{equation}
\IE r=\sum_{j=1}^K\Delta_j\Pr(j=\arg\max_i\overline X_i)
\label{eqn:simple-regret}
\end{equation}
where $\Delta_j=\mu_*-\mu_j$.
\end{dfn}

Strategies that minimize the simple regret are called pure exploration
strategies \cite{Bubeck.pure}. An upper bound on the simple regret of uniform sampling is
exponentially decreasing in the number of samples (see
\cite{Bubeck.pure}, Proposition~1). For UCB($c$) the best known respective upper bound on the simple
regret of UCB($c$) is only polynomially decreasing in the number of samples
(see \cite{Bubeck.pure}, Theorems~2,3). However, empirically
UCB($c$) appears to yield a lower simple regret than uniform
sampling. 


\subsection{Metareasoning}

A completely different scheme for control of sampling can use the
principles of bounded rationality \cite{Horvitz.reasoningabout}
and metareasoning ---  \cite{Russell.right} provided a formal
description of rational metareasoning and case studies of applications
in several problem domains. In search, under myopic and sub-tree independence assumptions,
one maintains a current best move  $\alpha $ at the root, 
and finds the expected gain from finding another move 
$\beta $ to be better than the current best \cite{Russell.right}. The ``cost'' of search
actions can also be factored in.
Ideally, an ``optimal'' sampling scheme, to be used
for selecting what to sample, both at the root node \cite{HayRussell.MCTS} and elsewhere,
can be developed using metareasoning.
However, this task is daunting for the following reasons:
\begin{itemize}
\item The method is in general intractable, necessitating simplifying assumptions.
 However, using the standard metareasoning myopic assumption,
    where samples would be selected as though at most one sample
   can be taken before an action is chosen, we run into serious problems. Even the basic
  selection problem \cite{TolpinShimony.blinkered} exhibits a
  non-concave utility function and results in premature stopping of the
  standard myopic algorithms. This is due to the fact that the value of information of
  a single measurement (analogous to a sample in MCTS) is frequently
  less than its time-cost, even though this is not true for multiple
  measurements. 

When applying the selection problem to MCTS, the
situation is exacerbated.  The utility of an action is usually
bounded, and thus in many cases a single sample may be insufficient
to change the current best action, \emph{regardless} of its
outcome. As a result, we frequently get a \emph{zero} ``myopic''
value of information for a single sample.
\item Rational metareasoning requires a known distribution model, which may be
  difficult to obtain.
\item Defining the time-cost of a sample is not trivial.
\end{itemize}

As the above ultimate goal is extremely difficult
to achieve, we introduce in this paper
simple schemes more amenable to analysis, loosely based on the metareasoning
concept of value of information, and compare
them to UCB (on sets) and UCT (in trees).

\section{Sampling Based on Simple Regret}
\label{sec:results}

\subsection{Analysis of Sampling on Sets}
\label{sec:sampling-on-sets}

We examine two sampling schemes with super-polynomially
decreasing upper bounds on the simple regret. The bounds
suggest that these schemes achieve a lower simple regret
than uniform sampling; indeed, this is confirmed
by experiments. 

We first consider $\varepsilon$-greedy sampling as a straightforward
generalization of uniform sampling:
\begin{dfn} The \textbf{$\mathbf{\varepsilon}$-greedy} sampling scheme
pulls the arm that currently has the greatst sample mean, with probability
$0<\varepsilon<1$, and any other
arm with probability $\frac {1-\varepsilon} {K-1}$. 
\end{dfn}
This sampling scheme exhibits an exponentially decreasing simple regret:
\begin{thm} For every $0<\eta<1$  and $\gamma>1$ there exists $N$ such that for
  any number of samples $n>N$ the simple regret of the  $\varepsilon$-greedy
sampling scheme is  bounded from above as
\begin{equation}
\IE r_{\varepsilon\mbox{-}greedy}\le 2\gamma \sum_{i=1}^K\Delta_i\exp\left(\frac {-2\Delta_j^2n\varepsilon}
  {\left(1+\sqrt{\frac {(K-1)\varepsilon}
        {1-\varepsilon}}\right)^2}\right)
\end{equation}
with probability at least $1-\eta$.
\end{thm}

\begin{proof}[Proof outline:] 

Bound the probability $P_i$ that a non-optimal arm $i$ is selected. Split the interval
 $[\mu_i, \mu_*]$ at $\mu_i+\delta_i$. Apply the Chernoff-Hoeffding bound to get:
\begin{eqnarray}
P_i&\le&\Pr[\overline X_i>\mu_i+\delta_i]+\Pr[\overline X_*<\mu_*-(\Delta_i-\delta_i)]\nonumber\\
   &\le&\exp\left(-2\delta_i^2n_i\right)+\exp\left(-2(\Delta_i-\delta_i)^2n_*\right)
\end{eqnarray}
Observe that, in probability, $\overline X_i \rightarrow \mu_i$ as $n\rightarrow\infty$, 
therefore $n_*\rightarrow n\varepsilon$, $n_i\rightarrow\frac
{n(1-\varepsilon)} {K-1}$ as $n\rightarrow \infty$. Conclude that for
every $0<\eta<1$, $\gamma>1$ there exists $N$ such that for every $n>N$ and
all non-optimal arms $i$:
\begin{equation}
P_i \le \gamma \left(\exp\left(\frac {-2\delta_i^2n(1-\varepsilon)}{K-1}\right)
+\exp\left(-2(\Delta_i-\delta_i)^2n\varepsilon\right)\right)
\label{eqn:epsgreedy-probbound}
\end{equation}
Require
\begin{eqnarray}
\exp\left(-\frac {2\delta_i^2n(1-\varepsilon)}{K-1}\right)
&=&\exp\left(-2(\Delta_i-\delta_i)^2n\varepsilon\right)\nonumber\\
\frac {\delta_i} {\Delta_i-\delta_i}&=&\sqrt{\frac {(K-1)\varepsilon} {1-\varepsilon}}
\label{eqn:epsgreedy-constant-varepsilon}
\end{eqnarray}
Substitute (\ref{eqn:epsgreedy-probbound}) together with
(\ref{eqn:epsgreedy-constant-varepsilon}) into
(\ref{eqn:simple-regret}) and obtain
\begin{equation}
\IE r_{\varepsilon\mbox{-}greedy}\le 2\gamma \sum_{i=1}^K\Delta_i\exp\left(\frac {-2\Delta_i^2n\varepsilon}
  {\left(1+\sqrt{\frac {(K-1)\varepsilon}
        {1-\varepsilon}}\right)^2}\right)
\end{equation}
\end{proof}

In particular, as the number of arms $K$ grows, the bound for $\frac 1
2$-greedy sampling ($\varepsilon=\frac 1 2$) becomes considerably tighter than for uniform
random sampling ($\varepsilon=\frac 1 K$):
\begin{crl}
For uniform random sampling, 
\begin{equation}
\IE r_{uniform}\le 2\gamma \sum_{i=1}^K\Delta_i\exp\left(-\frac {\Delta_i^2n} {K}\right)
\end{equation}
For $\frac 1 2$-greedy sampling,
\begin{eqnarray}
\IE r_{\frac 1 2\mbox{-}greedy}&\le& 2\gamma \sum_{i=1}^K\Delta_i\exp\left(\frac {-2\Delta_i^2n}
  {\left(1+\sqrt{K-1}\right)^2}\right)\\
  &\approx& 2\gamma \sum_{i=1}^K\Delta_i\exp\left(\frac
    {-2\Delta_i^2n} {K}\right)\mbox{ for }K\gg 1\nonumber
\end{eqnarray}
\end{crl}

$\varepsilon$-greedy is based solely on sampling the
arm that has the greatest sample mean (henceforth called the ``current best'' arm)
with a higher probability then the rest of the arms, and ignores
information about sample means of other arms. On the other hand,
UCB distributes samples in accordance with sample means, but, in order to
minimize cumulative regret, chooses the current best arm too often.
Intuitively, a better scheme for simple regret minimization would
distribute samples in a way similar to UCB, but would sample the current best arm
less often. This can be achieved by replacing $\log(\cdot)$ in
Equation~\ref{eqn:ucb} with a faster growing sublinear function, for
example, $\sqrt\cdot$.
\begin{dfn} Scheme $\mathbf{UCB_{\sqrt{\cdot}}(c)}$ pulls arm $i$ that
maximizes $b_i$, where:
\begin{equation}
b_i=\overline X_i+\sqrt {\frac {c \sqrt n} {n_i}}
\end{equation}
where, as before, $\overline X_i$ is the average reward obtained from arm $i$,
$n_i$ is the number of times arm $i$ was pulled, and $n$ is the total
number of pulls so far. \end{dfn}
This scheme also exhibits a super-polynomially decreasing simple regret:
\begin{thm}  For every $0<\eta<1$  and $\gamma>1$ there exists $N$ such that for
  any number of samples $n>N$ the simple regret of the  UCB$_{\sqrt{\cdot}}$($c$)
sampling scheme is  bounded from above as
\begin{equation}
\IE r_{ucb\sqrt{\cdot}} \le
2\gamma\sum_{i=1}^K\Delta_i\exp\left(-\frac {c\sqrt{n}} 2\right)
\end{equation}
with probability at least $1-\eta$.
\end{thm}

\begin{proof}[Proof outline:] Bound the probability $P_i$ that a
  non-optimal arm $i$ is chosen. Split the interval $[\mu_i, \mu_*]$
  at $\mu_i+\frac {\Delta_i} 2$. Apply the Chernoff-Hoeffding bound to get:
\begin{eqnarray}
P_i&\le&Pr\left[\overline X_i>\mu_i+\frac {\Delta_i} 2\right]+\Pr\left[\overline X_*<\mu_*-\frac {\Delta_i}
  2\right]\nonumber \\
   &\le&\exp\left(-\frac {\Delta_i^2n_i} 2\right)+\exp\left(-\frac {\Delta_i^2n_*} 2\right)
\end{eqnarray}
Observe that, in probability, $n_i\to \frac {c \sqrt n}
{\Delta_i^2}$, $n_i\le n_*$ as $n\to\infty$. Conclude that for
every $0<\eta<1$, $\gamma>1$ there exists $N$ such that for every $n>N$ and
all non-optimal arms $i$:
\begin{equation}
P_i \le 2\gamma \exp\left(-\frac {c \sqrt n} 2\right)
\label{eqn:usbsqrt-probbound}
\end{equation}
Substitute (\ref{eqn:usbsqrt-probbound}) into
(\ref{eqn:simple-regret}) and obtain
\begin{equation}
\IE r_{ucb\sqrt{\cdot}} \le 2\gamma\sum_{i=1}^K\Delta_i\exp\left(-\frac {c\sqrt{n}} 2\right)
\end{equation}
\end{proof}

\subsection{Sampling in Trees}
\label{sec:sampling-in-trees}

As mentioned above,
UCT \cite{Kocsis.uct} is an extension of UCB for MCTS, that
applies UCB($c$) at each step of a rollout.  At the root node,
the sampling in MCTS is usually aimed at finding the first move to perform. Search is re-started,
either from scratch or using some previously collected information,
after observing the actual outcome (in MDPs) or the opponent's move
(in adversarial games). Once one move is shown to be the best
choice with high confidence, the value of information of additional
samples of the best move (or, in fact, of any other samples) is low. Therefore, one should be able to do
better than UCT by optimizing {\em simple regret}, rather than {\em cumulative regret}, at the root node.

Nodes deeper in the search tree are a different matter.
In order to support an optimal move choice at the root, it is beneficial in many cases
to find a more precise estimate of the {\em value} of the state in these search
tree nodes. For these internal nodes, optimizing simple regret is not the answer, and 
cumulative regret optimization is not so far off the mark. Lacking a complete metareasoning for sampling,
which would indicate the optimal way to sample both root nodes and internal nodes,
our suggested improvement to UCT thus combines different sampling schemes on the first step and
during the rest of each rollout:
\begin{dfn}
The \textbf{SR+CR MCTS sampling scheme} selects an action at the
current root node according to a scheme suitable for minimizing 
the simple regret (\textbf{SR}), such as $\frac 1 2$-greedy or UCB$_{\sqrt{\cdot}}$, and (at non-root nodes)
then selects actions according to UCB, which approximately minimizes the cumulative regret (\textbf{CR}).
\end{dfn}

The pseudocode of this two-stage rollout for an undiscounted MDP is in
Algorithm~\ref{alg:two-stage-mcts}: \textsc{FirstAction} selects the first
step of a rollout (line~\ref{alg:srcr-first-action}), and
\textsc{NextAction} (line~\ref{alg:srcr-next-action}) selects steps during
the rest of the rollout (usually using UCB). The reward statistic for
the selected action is updated (line~\ref{alg:srcr-update-stat}), and
the sample reward is back-propagated (line~\ref{alg:return-reward}) 
towards the current root.

We denote such two-step realizations of SR+CR as \emph{Alg}+UCT, where
\emph{Alg} is the sampling scheme employed at the first step of a
rollout (e.g. $\frac 1 2$-greedy+UCT).

\begin{algorithm}[h!]
\caption{Two-stage Monte-Carlo tree search sampling}
\label{alg:two-stage-mcts}
\begin{algorithmic}[1]
\Procedure{Rollout}{node, depth=1}
  \If {\Call{IsLeaf}{node, depth}}
    \State \textbf{return} 0
  \Else
    \State \textbf{if} depth=1 \textbf{then} action $\gets$ \Call{FirstAction}{node} \label{alg:srcr-first-action}   
    \State \textbf{else} action $\gets$ \Call{NextAction}{node} \label{alg:srcr-next-action}
    \State next-node $\gets$ \Call{NextState}{node, action}
    \State reward $\gets$ \Call{Reward}{node, action, next-node}
     \State \hspace{4em} + \Call{Rollout}{next-node, depth+1}
    \State \Call{UpdateStats}{node, action, reward} \label{alg:srcr-update-stat}
    \State \textbf{return} reward \label{alg:return-reward}
  \EndIf
\EndProcedure
\end{algorithmic}
\end{algorithm}

We expect such two-stage sampling schemes to
outperform UCT and be
significantly less sensitive to the tuning of the exploration factor
$c$ of UCB($c$). That is since the contradiction between the need for a larger
value of $c$ on the first step (simple regret) and a smaller
value for the rest of the rollout (cumulative regret)
\cite{Bubeck.pure} is resolved. In fact, a sampling scheme that uses
UCB($c$) at all steps but a larger value of $c$ for the first step
than for the rest of the steps, should also outperform UCT.

\subsection{VOI-aware Sampling}
\label{sec:voi-sampling}

Further improvement can be achieved by computing or estimating the
value of information (VOI) of the rollouts and choosing rollouts that
maximize the VOI. 
However, as indicated above, actually computing the VOI is infeasible.
Instead we suggest the following scheme based on the following
features of value of information:
\begin{enumerate}
\item An estimate of the probability that one or more rollouts will make another action
appear better than the current best $\alpha $.
\item An estimate of the gain that may be incurred if such a change occurs.
\end{enumerate}

If the distribution of results generated by the rollouts were known, the
above features could be easily computed. However, this is not the case for
most MCTS applications. Therefore, we estimate bounds on the feature values from
the current set of samples, based on the myopic assumption that the algorithm will only
sample one of the actions, and use these bounds as the feature values, to get:
\begin{eqnarray}
VOI_\alpha&\approx&\frac {\overline X_\beta} {n_\alpha+1}
\exp\left(-2(\overline X_\alpha - \overline X_\beta)^2 n_\alpha\right)\\
VOI_i&\approx&\frac {1-\overline X_\alpha} {n_i+1}
\exp\left(-2(\overline X_\alpha - \overline X_i)^2 n_i\right),\; i\ne\alpha\nonumber\\
\mbox{where }&&\alpha=\arg\max_i \overline X_i,\quad
             \beta=\arg\max_{i,\,i\ne\alpha} \overline X_i\nonumber
\end{eqnarray}
with $VOI_{\alpha}$ being the (approximate) value for sampling the current best action,
and $VOI_i$ is the (approximate) value for sampling some other action $i$. 

These equations were derived as follows. The gain from switching from the current best action $\alpha$ to another
action can be bounded:  by  the current expectation of the value the current second-best action
for the case where we sample only $\alpha$, and by 1 (the maximum reward) minus the current expectation
of $\alpha $ when sampling any other action. 
The probability that another action be found
best can be bounded by an exponential function of the difference in expectations when the true value
of the actions becomes known. But the effect of each individual sample on the sample mean
is inversely proportional to the current number of samples, hence the current number of samples (plus one
in order to handle the initial case of no previous samples) in the denominator.

These VOI estimates are used in the ``VOI-aware'' sampling scheme as follows: sample the
action that has maximum estimated VOI. We judged these estimates to be too crude
to be used as ``stopping criteria'' that can be used to cut off sampling,
leaving this issue for future research.
Although this scheme appears too complicated to be amenable to a formal
analysis, early experiments (Section \ref{sec:emp})
with this approach demonstrate a significantly lower simple regret.

\section{Empirical Evaluation}
\label{sec:emp}

The results were empirically verified on Multi-armed Bandit instances,
on search trees, and on the sailing domain, as defined in
\cite{Kocsis.uct}. In most cases, the experiments showed a lower average
simple regret for $\frac 1 2$-greedy an UCB$_{\sqrt{\cdot}}$ than for
UCB on sets, and for the SR+CR scheme than for UCT in trees.

\subsection{Simple regret in multi-armed bandits}
\label{sec:emp-mab}

Figure~\ref{fig:mab-simple-regret} presents a comparison of MCTS sampling
schemes on Multi-armed bandits. Figure~\ref{fig:mab-simple-regret}.a shows the search tree
corresponding to a problem instance. Each arm returns a random reward
drawn from a Bernoulli distribution. The search selects an arm
and compares the expected reward, unknown to the algorithm during the
sampling, to the expected reward of the best arm.

Figure~\ref{fig:mab-simple-regret}.b shows the regret
vs. the number of samples, averaged over $10000$ experiments for
randomly generated instances of 32 arms. Either $\frac 1
2$-greedy  or UCB$_{\sqrt{\cdot}}$ dominate UCB over the
whole range. For larger number of samples, the advantage
of UCB$_{\sqrt{\cdot}}$ over $\frac 1 2$-greedy becomes more significant.

\begin{figure}[h!]
  \begin{minipage}[c]{1.0\linewidth}
    \centering
    \includegraphics[scale=0.8]{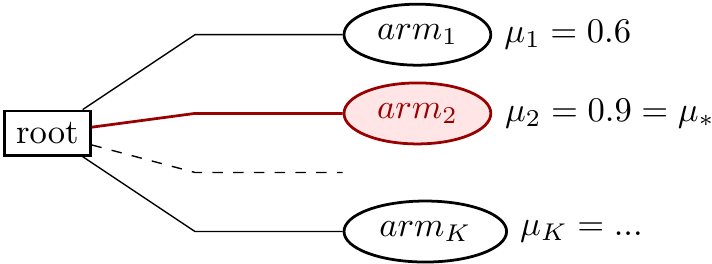}\\
    a. search tree, with best arm shaded
    \vspace{1em}
  \end{minipage}
  \begin{minipage}[c]{1.0\linewidth}
    \centering
    \includegraphics[scale=0.45]{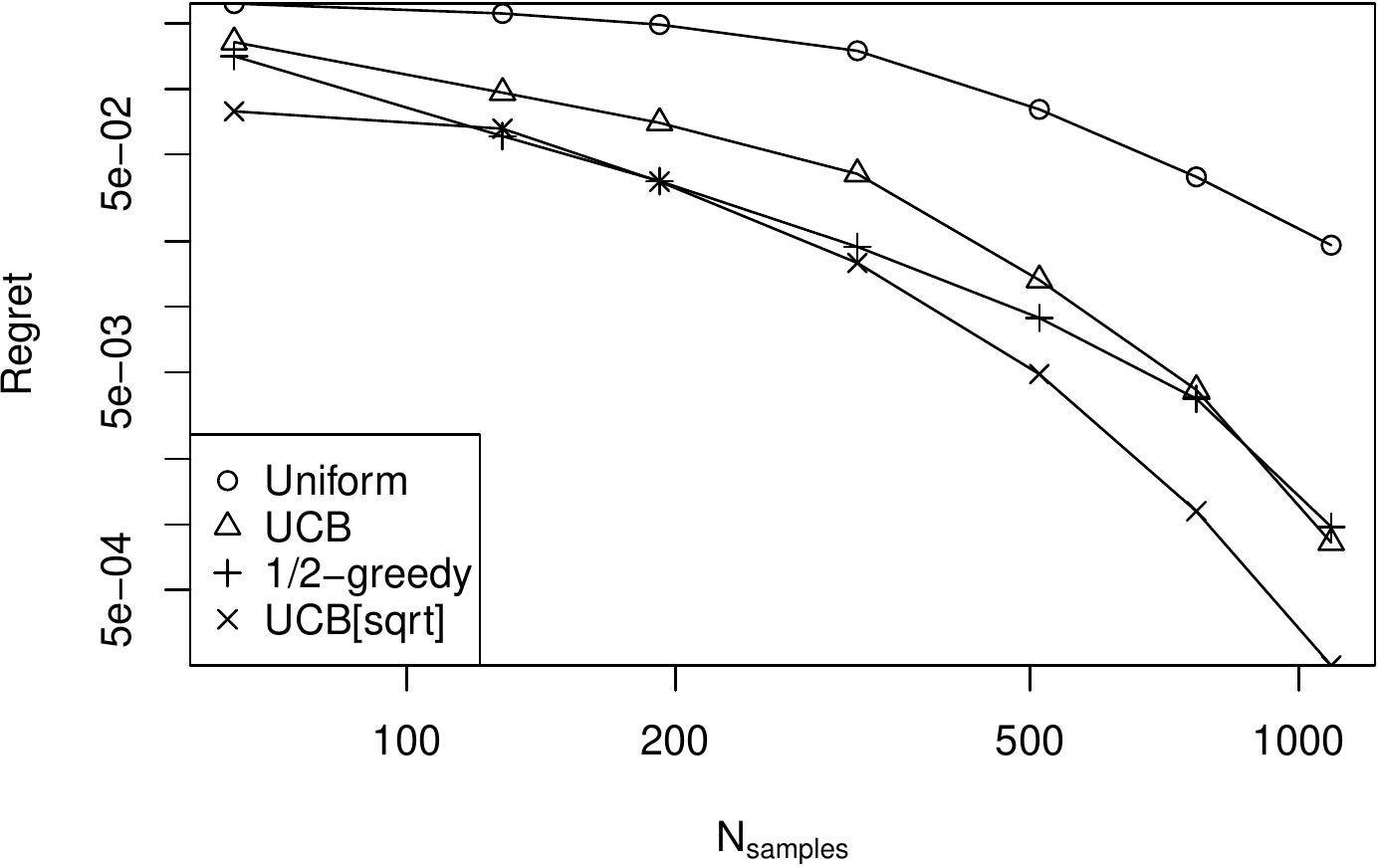}\\
    b. regret vs. number of samples
  \end{minipage}
  \caption{Simple regret in MAB}
  \label{fig:mab-simple-regret}
\end{figure}

\subsection{Monte Carlo tree search}
\label{sec:emp-mcts}

The second set of experiments was performed on randomly generated
2-level max-max trees crafted so as to deliberately deceive uniform sampling
(Figure~\ref{fig:mcts-regret}.a), necessitating  an adaptive sampling scheme, 
such as UCT. That is due to the switch nodes, each with 2 children with anti-symmetric
values, which would cause a uniform sampling scheme to incorrectly give them all a value of 0.5.

Simple regret vs. the number of samples are shown for trees with root degree 16
(Figure~\ref{fig:mcts-regret}.b) and 64
(Figure~\ref{fig:mcts-regret}.c). The exploration factor $c$ is set to
$2$, the default value for rewards in the range $[0, 1]$.
The algorithms exhibit a similar relative performance: either $\frac 1
2$-greedy+UCT or UCB$_{\sqrt{\cdot}}$+UCT
result in the lowest regret, UCB$_{\sqrt{\cdot}}$+UCT dominates UCT everywhere
except when the number of samples is small. The advantage of both $\frac 1
2$-greedy+UCT and UCB$_{\sqrt{\cdot}}$+UCT grows with the number of arms.

\begin{figure}[h!]
  \begin{minipage}[c]{1.0\linewidth}
    \centering
    \includegraphics[scale=0.7]{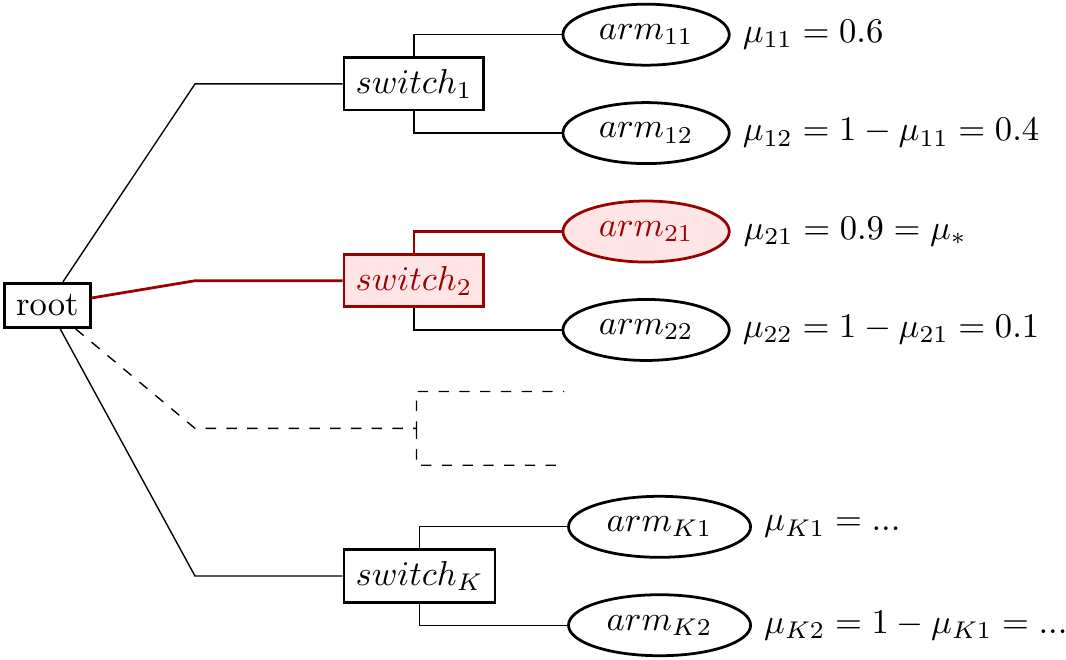}\\
    a. search tree, with path to the best arm shaded
    \vspace{1em}
  \end{minipage}
  \begin{minipage}[c]{1.0\linewidth}
    \centering
    \includegraphics[scale=0.45]{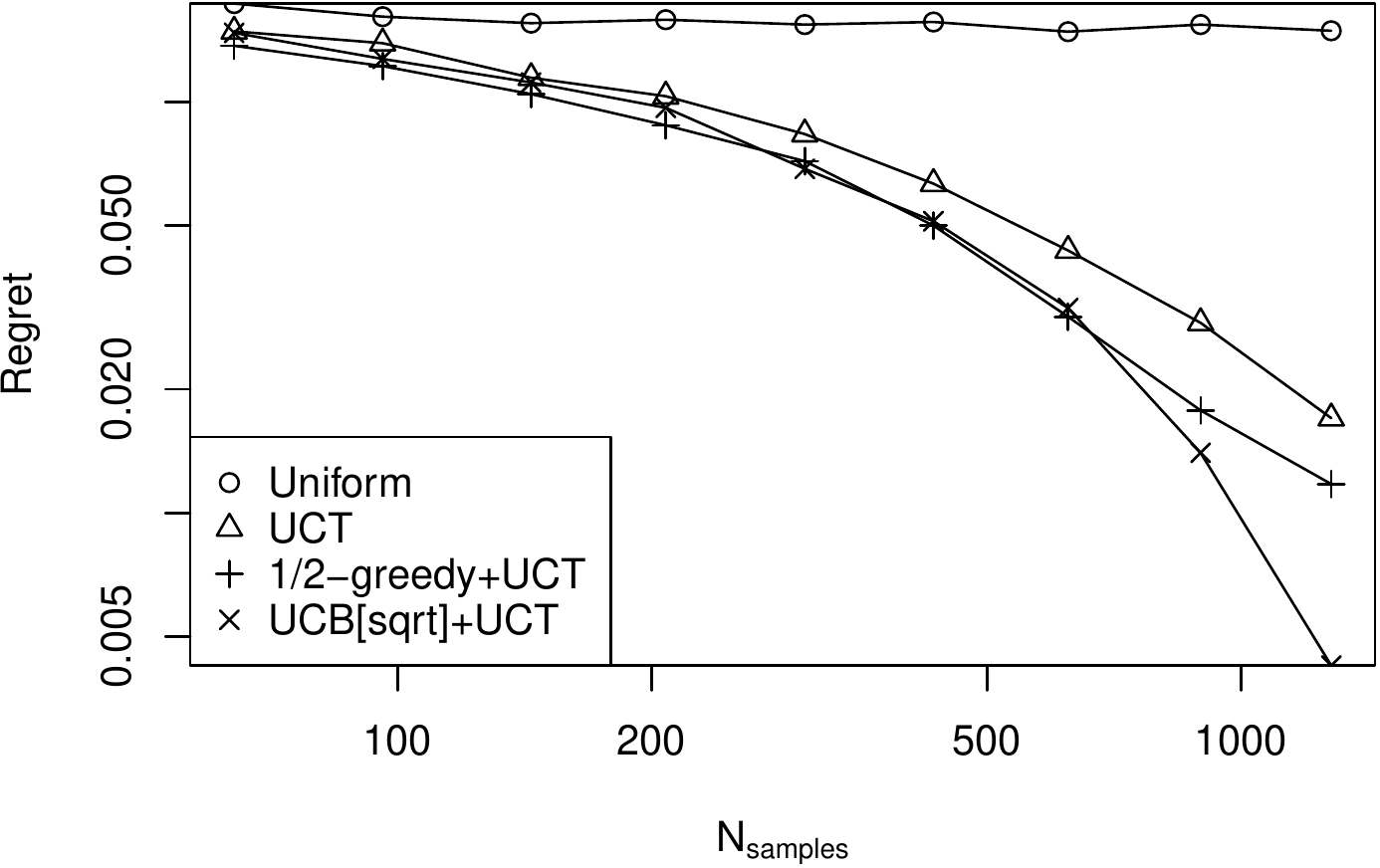}\\ 
    b. 16 arms
    \vspace{0.5em}
  \end{minipage}
  \begin{minipage}[c]{1.0\linewidth}
    \centering
    \includegraphics[scale=0.45]{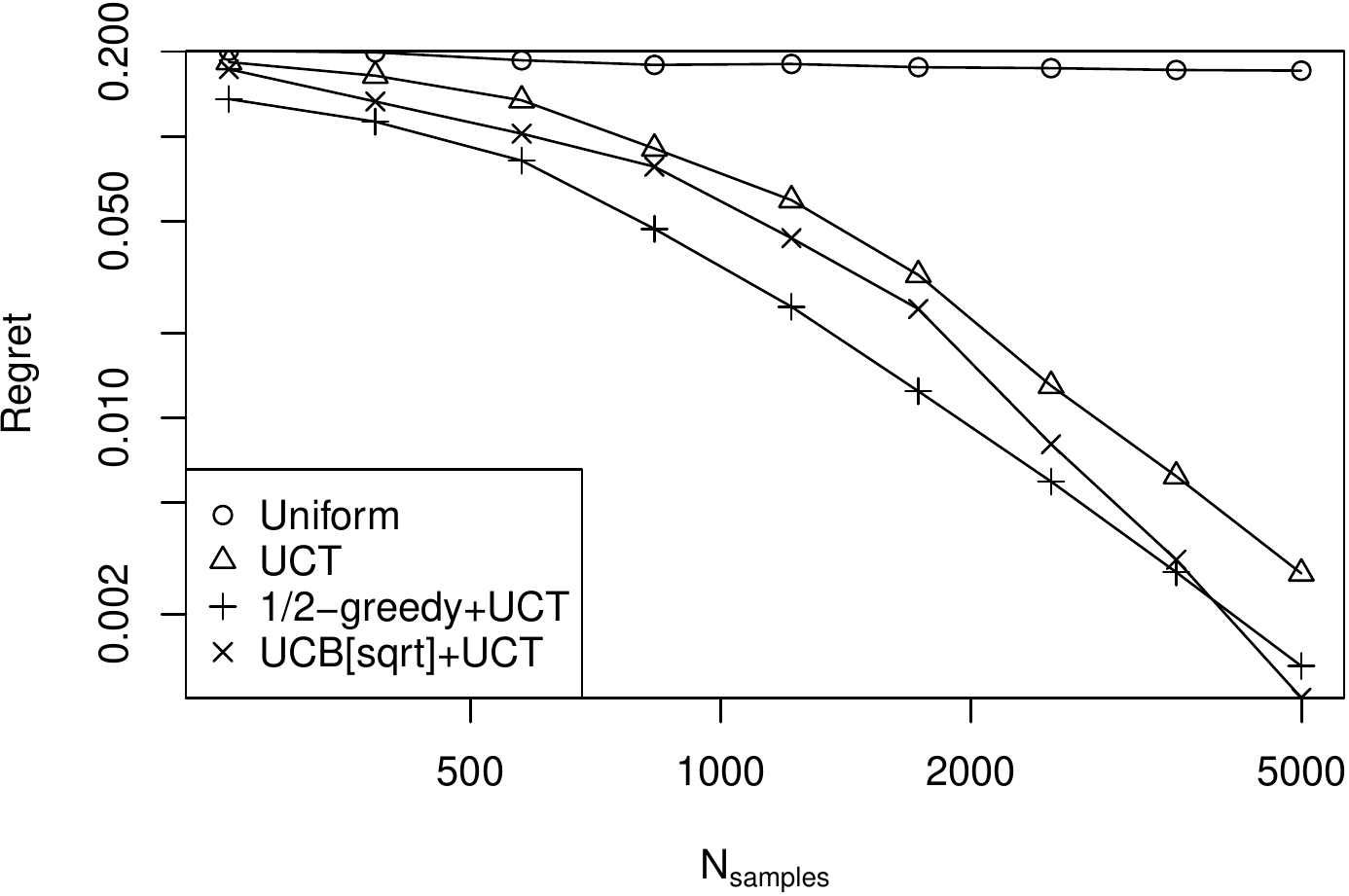} \\
    c. 64 arms
  \end{minipage}
  \caption{MCTS in random trees}
  \label{fig:mcts-regret}
\end{figure}

\subsection{The sailing domain}
\label{sec:emp-sailing}

Figures~\ref{fig:sailing-cost-vs-nsamples}--\ref{fig:sailing-lake-size}
show results of experiments on the sailing
domain. Figure~\ref{fig:sailing-cost-vs-nsamples} shows the regret
vs. the number of samples, computed for a range of values of
$c$. Figure~\ref{fig:sailing-cost-vs-nsamples}.a shows the median
cost, and Figure~\ref{fig:sailing-cost-vs-nsamples}.b --- the minimum
costs. UCT is always worse than either $\frac 1 2$-greedy+UCT or
UCB$_{\sqrt{\cdot}}$+UCT, and is sensitive to the value of $c$: the median cost is
much higher than the minimum cost for UCT. For both $\frac 1 2$-greedy+UCT
and UCB$_{\sqrt{\cdot}}$+UCT, the difference is significantly less prominent.

\begin{figure}[h!]
  \begin{minipage}[b]{1.0\linewidth}
    \centering
    \includegraphics[scale=0.45]{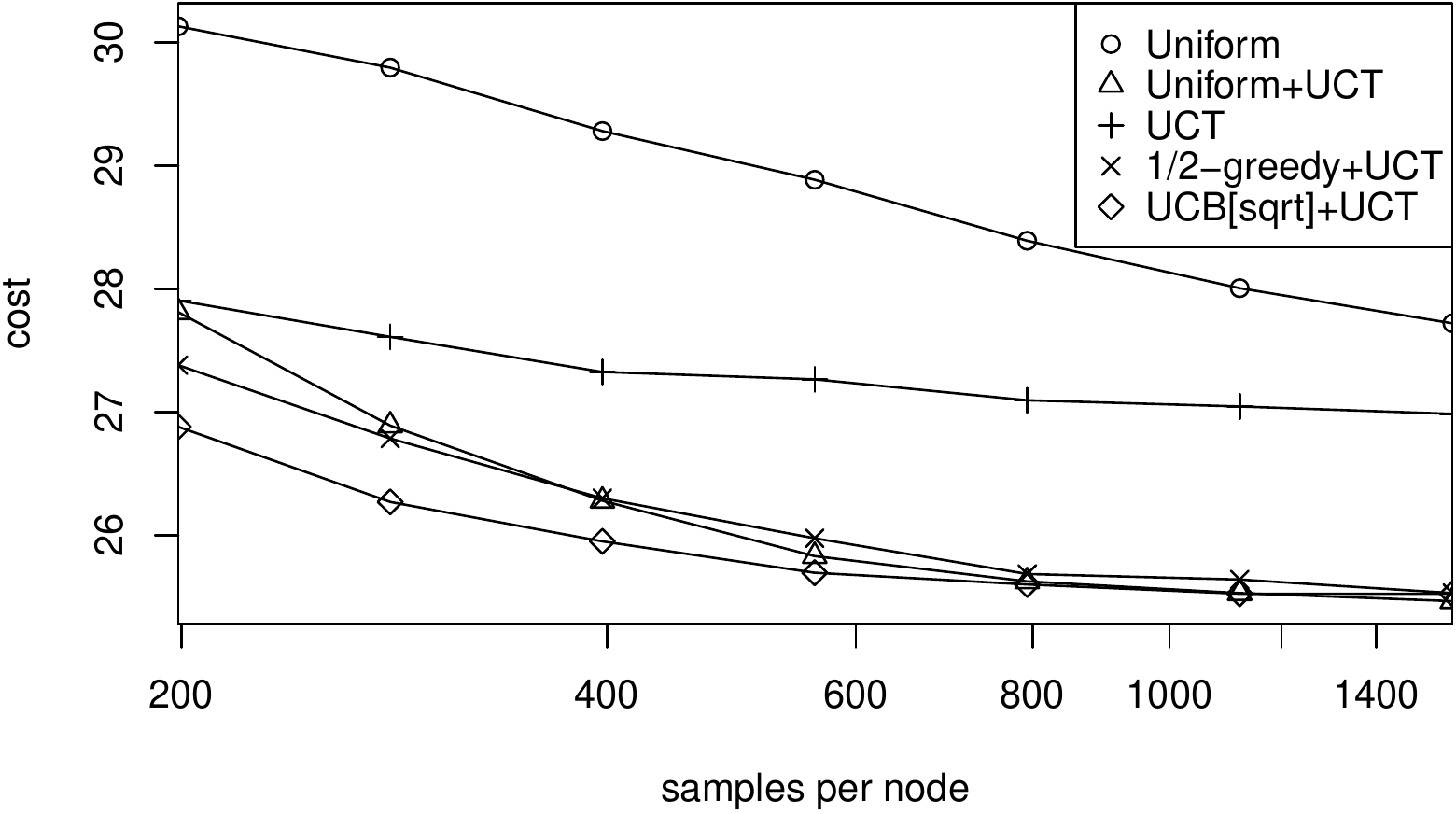}\\
    a. median cost
    \vspace{1em}
  \end{minipage}
  \begin{minipage}[b]{1.0\linewidth}
    \centering
    \includegraphics[scale=0.45]{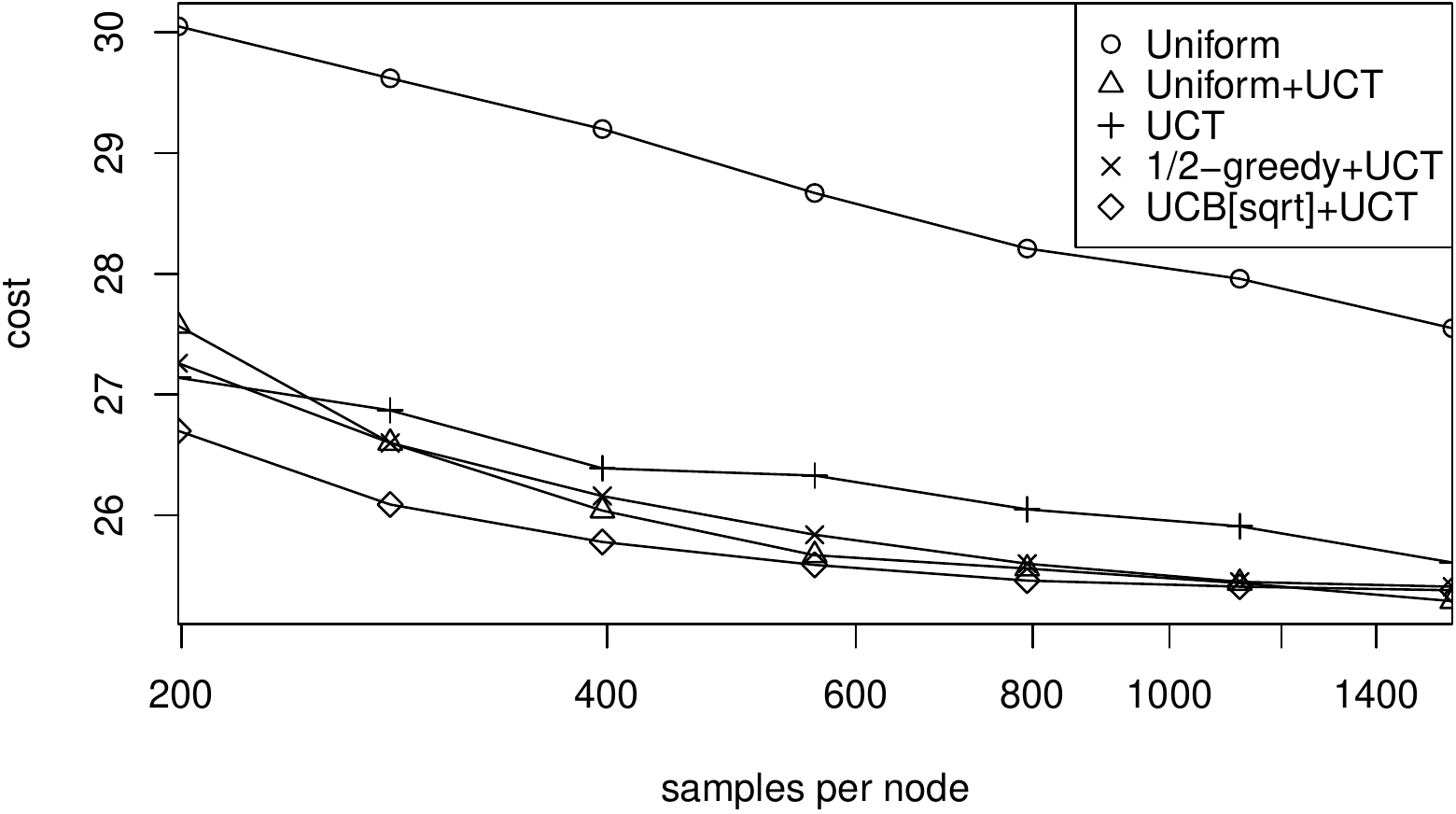}\\
    b. minimum cost
  \end{minipage}
  \caption{The sailing domain, $6\times 6$ lake, cost vs. samples}
  \label{fig:sailing-cost-vs-nsamples}
\end{figure}

Figure~\ref{fig:sailing-cost-vs-factor} shows the regret vs. the
exploration factor for different numbers of samples. UCB$_{\sqrt{\cdot}}$+UCT is always better than
UCT, and $\frac 1 2$-greedy+UCT is better than UCT expect for a small range of
values of the exploration factor. 

\begin{figure}[h!]
  \centering
  \includegraphics[scale=0.45]{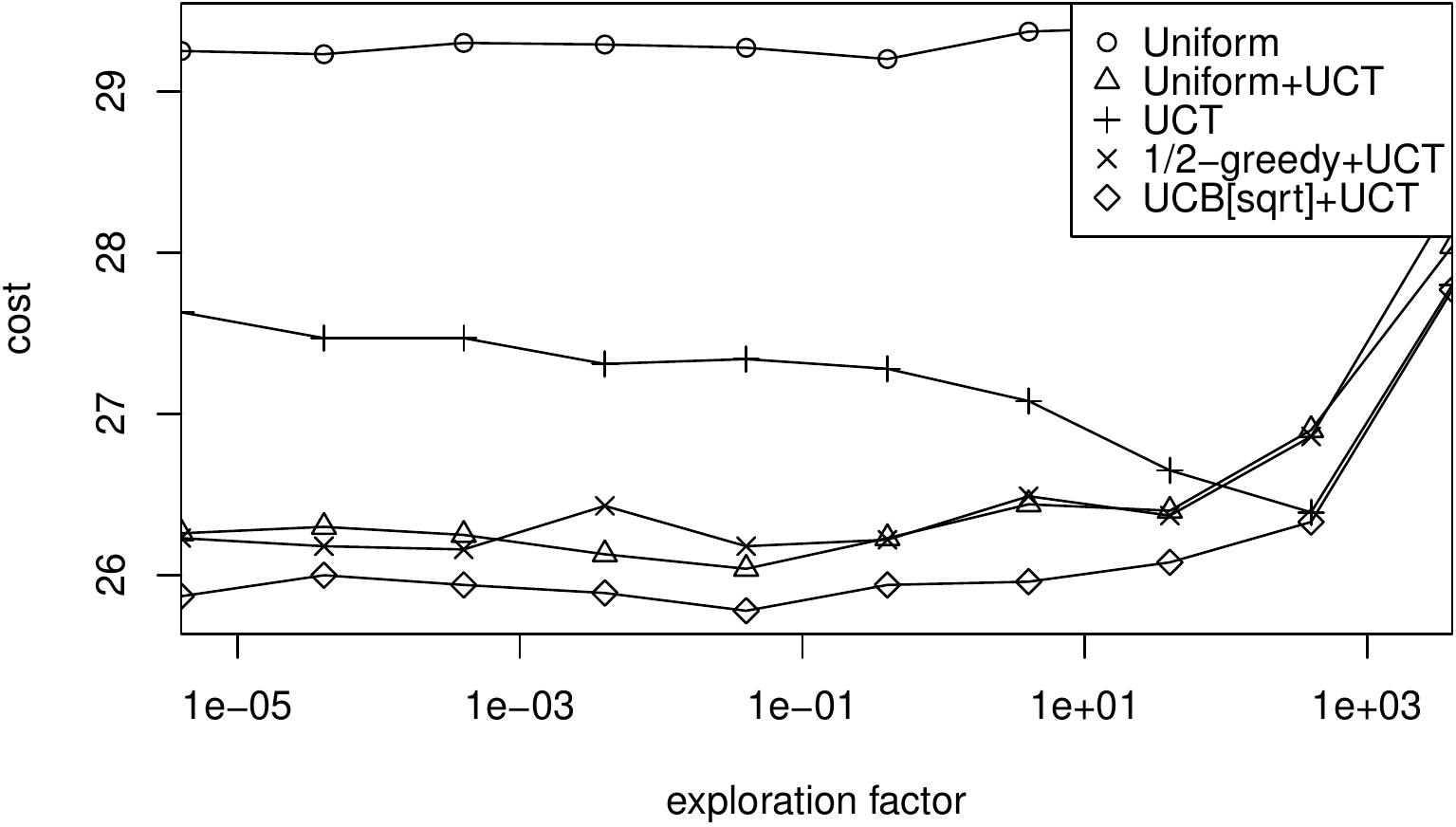}\\
  a. 397 rollouts\\
  \vspace{1em}
  \includegraphics[scale=0.45]{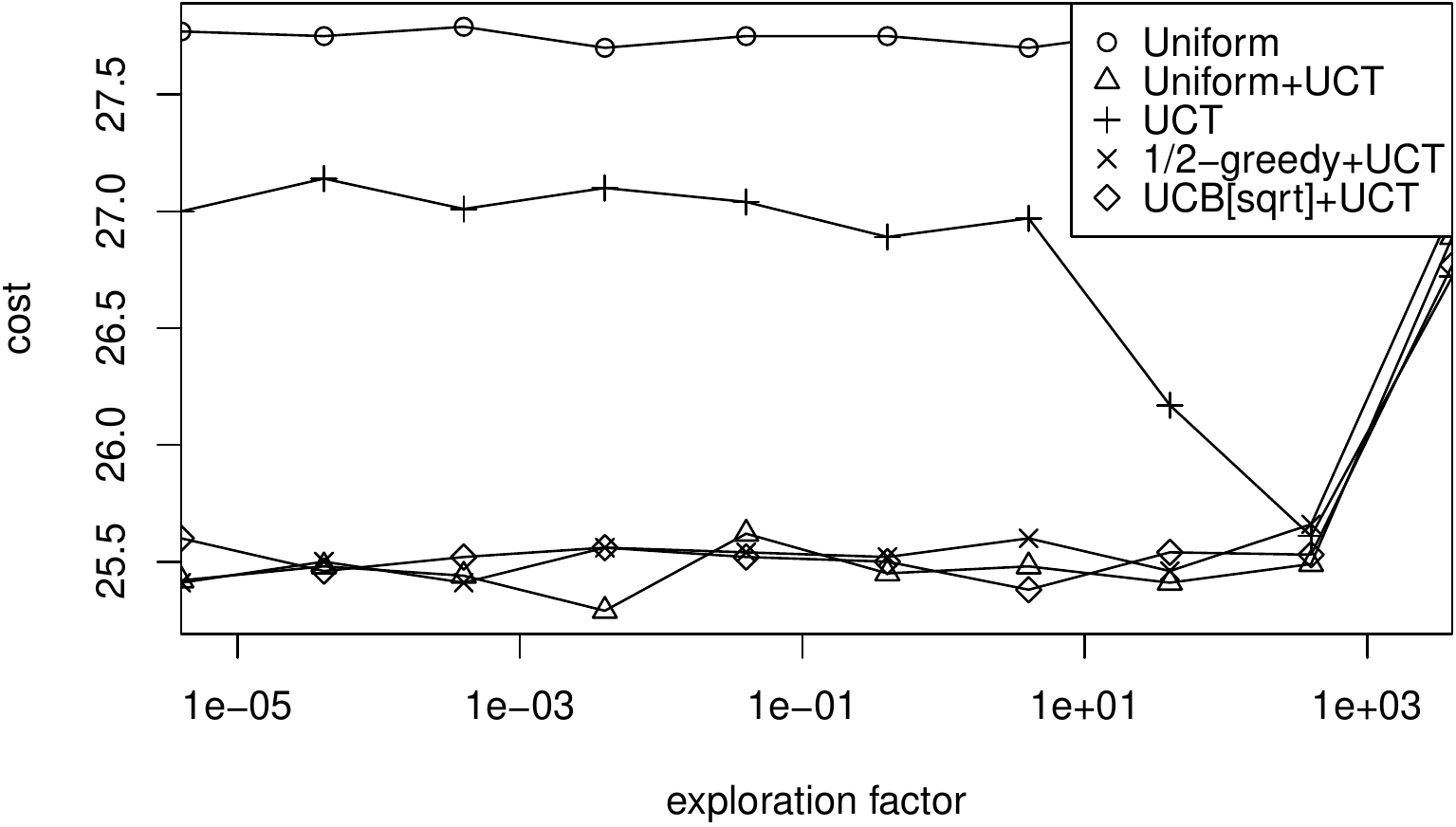}\\
  b. 1585 rollouts
  \caption{The sailing domain, $6\times 6$ lake, cost vs. factor}
  \label{fig:sailing-cost-vs-factor}
\end{figure}

Figure~\ref{fig:sailing-lake-size} shows the cost vs. the exploration
factor for lakes of different sizes. The relative difference between
the sampling schemes becomes more prominent when the lake size
increases.
\begin{figure}[h!]
   \centering
   \includegraphics[scale=0.45]{costs-size=6-nsamples=397.pdf}\\
   a. $6\times 6$ lake \\
   \vspace{1em}
   \includegraphics[scale=0.45]{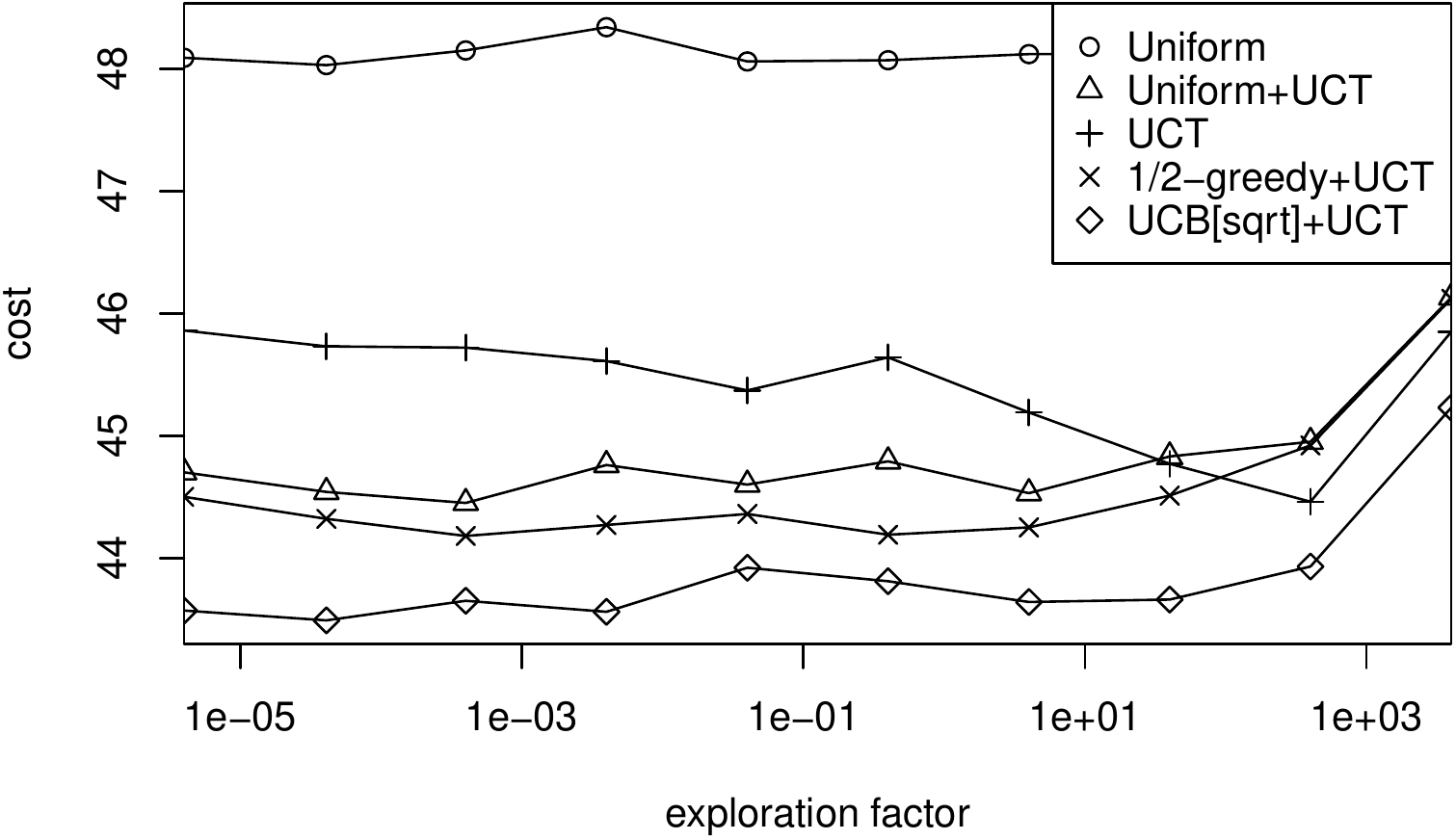}\\
   b. $10\times 10$ lake
  \caption{The sailing domain, 397 rollouts, cost vs. factor}
  \label{fig:sailing-lake-size}
\end{figure}

\subsection{VOI-aware MCTS}

Finally, the VOI-aware sampling scheme was empirically compared to
other sampling schemes (UCT, $\frac 1 2$-greedy+UCT,
UCT$_{\sqrt{\cdot}}$+UCT). Again, the experiments were performed on
randomly generated trees with structure shown in
Figure~\ref{fig:mcts-regret}.a. Figure~\ref{fig:mcts-regret-voi} shows
the results for 32 arms. VOI+UCT, the scheme based on a VOI estimate,
outperforms all other sampling schemes in this example. Similar performance improvements
(not shown) also occur for the sailing domain.
\begin{figure}[h!]
  \centering
  \includegraphics[scale=0.45]{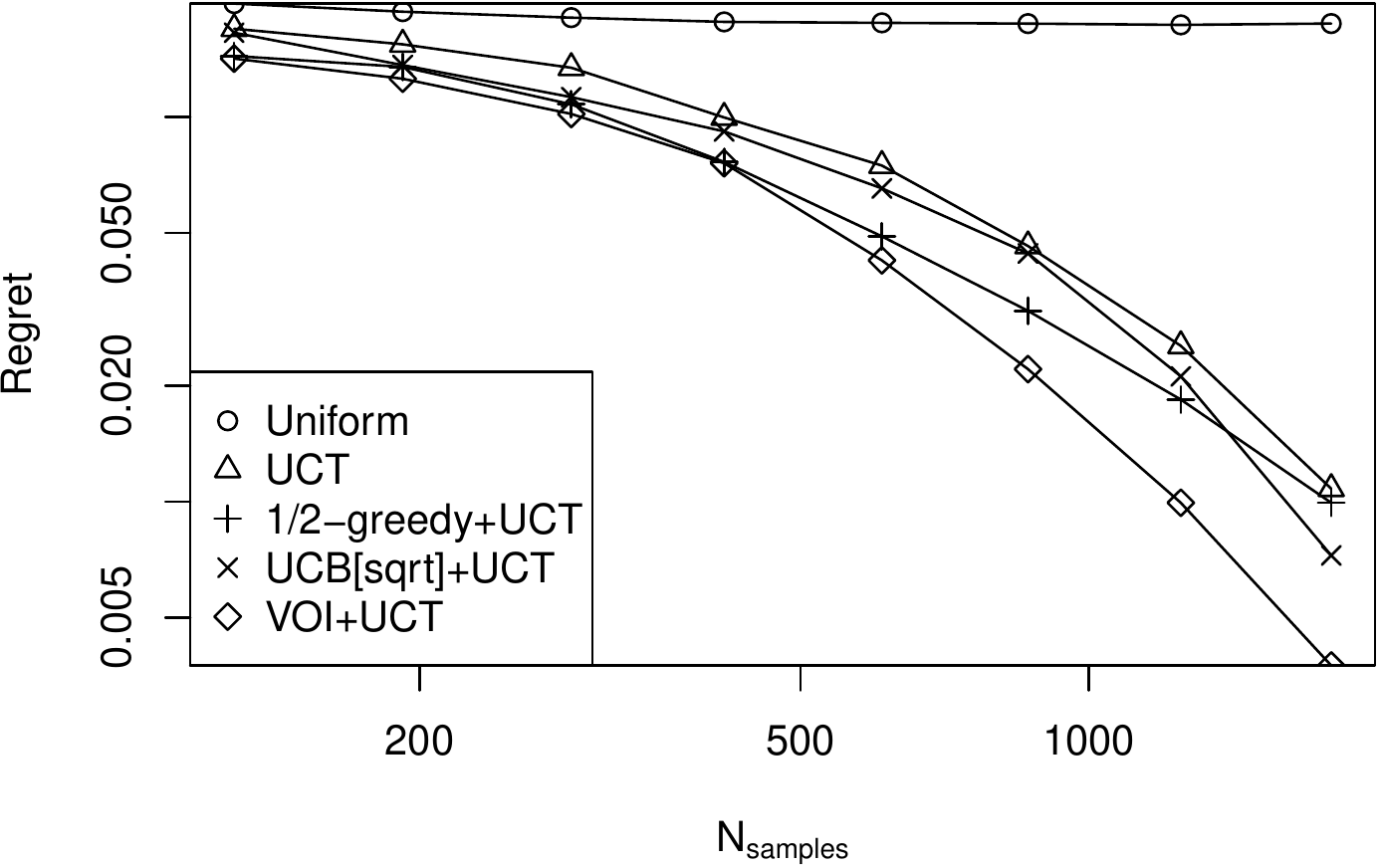}\\
  \caption{MCTS in random trees, including VOI+UCT.}
  \label{fig:mcts-regret-voi} 
\end{figure}

\section{Conclusion and Future Work}
\label{sec:summary}

UCT-based Monte-Carlo tree search has been shown to be very effective
for finding good actions in both MDPs and adversarial games.
Further improvement of the sampling scheme is thus of interest in
numerous search applications. We argue that although UCT is already very efficient,
one can do better if the sampling scheme is considered from a metareasoning
perspective of value of information (VOI).

The MCTS SR+CR scheme presented in the paper differs
from UCT mainly in the first step of the rollout, when we attempt to minimize
the `simple' selection regret rather than the cumulative regret. Both the
theoretical analysis and the empirical evaluation provide evidence for
better general performance of the proposed scheme.

Although SR+CR is inspired by the notion of VOI,
the VOI is used there implicitly in the analysis of the algorithm,
rather than computed or learned explicitly in order to plan the
rollouts. Ideally, using VOI to control sampling ab-initio should do even better,
but the theory for doing that is still not up to speed. Instead we suggest
a  ``VOI-aware'' sampling scheme based on crude probability and value estimates,
which despite its simplicity already shows a marked improvement in minimizing regret.
However, application of the theory of rational metareasoning
to Monte Carlo Tree Search is an open problem \cite{HayRussell.MCTS},
and both a solid theoretical model and empirically efficient VOI
estimates need to be developed. 

Finding a better sampling scheme for non-root nodes,
as well as the root node, should also be possible.
Although cumulative regret does reasonably
well there, it is far from optimal, as meta-reasoning principles imply that an optimal scheme
for these nodes must be asymmetrical (e.g. it is not helpful to find out that the
value of the current best action is even better than previously believed).

Finally, applying VOI methods in complex deployed applications that already use
MCTS is a challenge that should be addressed in future work.
In particular, UCT is extremely successful in Computer
Go \cite{GellyWang.mogo,Braudis.pachi,Enzenberger.Fuego},
and the proposed scheme should be evaluated on this domain. This is non-trivial, since
Go programs typically use ``non-pure'' versions of UCT, extended with
domain-specific knowledge. For example, Pachi \cite{Braudis.pachi}
typically re-uses information from rollouts generated for earlier
moves, thereby violating our underlying assumption that information is
only used for selecting the current move.  In early experiments not
shown here (disallowing re-use of samples, admittedly not really a
fair comparison) the VOI-aware scheme apears to dominate UCT.
Nevertheless, it should also be possible to adapt the VOI-aware
schemes to take into account expected re-use of samples, another topic
for future research.

\section*{Acknowledgments}

The research is partially supported by Israel
Science Foundation grant 305/09, by the Lynne and William Frankel
Center for Computer Sciences, and by the Paul Ivanier Center for
Robotics Research and Production Management.

\bibliographystyle{aaai}
\bibliography{refs}

\end{document}